%% file: main.tex
\documentclass[11pt]{article}

\input{math_commands.tex}

\usepackage{fullpage}
\usepackage{amsmath}
\usepackage{algorithmic}
\usepackage{hyperref}       
\usepackage{url}            
\usepackage{booktabs}       
\usepackage{nicefrac}       
\usepackage{microtype}      
\usepackage{mathtools}
\usepackage{graphicx}
\usepackage{algorithm}
\usepackage{color}
\usepackage{multirow}
\usepackage{subcaption}
\usepackage{float}
\usepackage{amsthm} 

\usepackage{enumitem}

\def\OM{\texttt{PINE}}

  \def\rc{\color{black}}

  \def\mc{\color{black}}
  
  \def\rrc{\color{black}}
  \def\revisec{\color{black}}
\def\rn{\color{black}} 
\def\bn{\color{black}} 

\def\mn{\color{black}}

\def\R{{\mathbb{R}}}

\def\bx{\mathbf{x}}
\def\bz{\mathbf{z}}
\def\bA{\mathbf{A}}

\def\bX{\mathbf{X}}
\def\bbR{\mathbb{R}}

\def\ba{\mathbf{a}}
\def\about#1{\emph{[#1]}}
\def\symsum{\frac{1}{|\mathcal{S}|}
  \sum_{\substack{{\tiny \pi_1 \in S_{N_1}}{\tiny\dots}\\{\tiny\pi_K\in S_{N_K}}}}}
\def\setR{\bbR}
\def\bp{\mathbf{p}}
\def\bq{\mathbf{q}}
\def\bg{\mathbf{g}}

\def\nodev{\mathtt{v}}

  \newtheorem{theorem}{Theorem}[section]
  
  \newtheorem{lemma}[theorem]{Lemma}

  \newtheorem{definition}[theorem]{Definition}
  \def\thref#1{Theorem~\ref{#1}}

\def\leref#1{Lemma~\ref{#1}}

\begin{document}
\title{\OM: Universal Deep Embedding for Graph Nodes via Partial Permutation Invariant Set Functions }
\author{
  Shupeng~Gui\\
  University of Rochester\\
  \texttt{sgui2@ur.rochester.edu}\\
  \and
  Xiangliang~Zhang\\
  KAUST, Saudi Arabia\\
  \texttt{Xiangliang.Zhang@kaust.edu.sa}\\
  \and
  Pan~Zhong\\
  Iowa State University\\
  \texttt{pzhong@iastate.edu}\\
  \and
    Shuang~Qiu\\
    University of Michigan\\
    \texttt{qiush@umich.edu}\\
    \and
    Mingrui~Wu\\
    PlusAI\\
    \texttt{wu.mingrui@yahoo.com}\\
    \and
    Jieping~Ye\\
    University of Michigan\\
    \texttt{jieping@gmail.com}\\
    \and
    Zhengdao~Wang\\
    Iowa State University\\
    \texttt{zhengdao@iastate.edu}\\
    \and
    Ji~Liu\\
    Kwai Inc.\\
    \texttt{ji.liu.uwisc@gmail.com}
}
\maketitle


\begin{abstract}
Graph node embedding aims at learning a vector representation for all nodes given a graph. It is a central problem in many machine learning tasks (e.g., node classification, recommendation, community detection). The key problem in graph node embedding lies in how to define the dependence to neighbors. Existing approaches specify (either explicitly or implicitly) certain dependencies on neighbors, which may lead to loss of subtle but important structural information within the graph and other dependencies among neighbors. This intrigues us to ask the question: can we design a model to give the maximal flexibility of dependencies to each node's neighborhood. In this paper, we propose a novel graph node embedding method (named \textbf{\OM}) via a novel notion of \textit{partial permutation invariant set function}, to capture any possible dependence. Our method 1) can learn an \textit{arbitrary} form of the representation function from the neighborhood, without losing any potential dependence structures, and 2) is applicable to both homogeneous and heterogeneous graph embedding, the latter of which is challenged by the diversity of node types. Furthermore, we provide theoretical guarantee for the representation capability of our method for general homogeneous and heterogeneous graphs. Empirical evaluation results on benchmark data sets show that our proposed {\OM} method outperforms the state-of-the-art approaches on producing node vectors for various learning tasks of both homogeneous and heterogeneous graphs.\end{abstract}



\section{Introduction}\label{sec:introduction}

Graph node embedding (or graph node representation learning in some literature {\rc\cite{goyal2018graph}}) is to learn the numerical representation for each node in a graph by vectors in a Euclidean space, where the geometric relationship reflects the structure of the original graph.
Nodes that are ``close'' in the graph are embedded to have similar vector representations~\cite{survey-CAI}.
The learned node vectors benefit a number of graph analysis tasks, such as node classification~\cite{bhagat2011node}, link prediction~\cite{liben2007link}, community detection~\cite{fortunato2010community}, {\rc recommendation~\cite{yu2014personalized}}, and many others~\cite{survey-jure}.
   
A graph can be uniquely determined by defining the neighborhood. 
Therefore, the key issue for graph embedding lies on {\rn how to model the dependence of each node to its neighbors.}

Existing approaches mostly specify (either explicitly or implicitly) certain {\rn dependencies on neighbors.}
{\revisec
Deepwalk~\cite{perozzi2014deepwalk}, node2vec \cite{grover2016node2vec}, and their variants \cite{dong2017metapath2vec, Nguyen:2018:CDN:3184558.3191526} randomly generate a set of paths with a fixed length to learn the representation for each node, which implicitly defines the neighborhood and the dependence among nodes.
\cite{Arbitrary-Order-KDD18} utilizes the adjacency matrix to represent the neighborhood for every node in a graph and apply matrix factorization for the node embedding learning, which implicitly defines the linear dependence among nodes.
{\rc Neighborhood auto-encoders \cite{SDNE, DNGR,Graph2Gauss} use a neighborhood vector to represent the neighborhood relations for a node. The neighborhood vector contains a node's pairwise similarity to all the other nodes in a graph. Graph2Gauss~\cite{Graph2Gauss} embeds each node as a Gaussian distribution based on the graph knowledge. 
Deep neural networks
for graph representations (DNGR)~\cite{DNGR} uses stacked denoising auto-encoder to extract complex non-linear features for each node. 
Structural Deep Network Embedding method (SDNE)~\cite{SDNE} preserves the first-order and second-order proximity for each node in a graph via a semi-supervised auto-encoder learning model.}
Neural network based approaches such as graph convolutional networks (GCN) {\rc \cite{kipf2016semi}} and GraphSAGE {\rc \cite{GraphSAGE}} define fixed-depth neural network layers to capture the neighborhood information from one-step neighbors, two-step neighbors, up to $n$-step neighbors and they apply convolution-like functions on these neighbors as the aggregation strategy. Graph attention networks (GATs)~\cite{velickovic2017graph} and Attention-based Graph Neural
Network (AGNN) \cite{thekumparampil2018attention} employ attention mechanism when aggregating the neighbors. 
}

    {\revisec
        However, the way of pre-defining (no matter explicitly or implicitly) neighbors and dependence may cause subtle but important loss of structural information within the graph and dependence among neighbors. For example, the family of random walk based methods {\rc \cite{perozzi2014deepwalk, grover2016node2vec, dong2017metapath2vec}} ignore the influence of nodes out of the predefined length to the center node within the path. 
        GCN{\rc\cite{kipf2016semi}} restricts the form of the dependence on the neighbor nodes to a two-layer aggregation function {\mn and inherits considerable complexity from their deep learning lineage \cite{simplifiedGCN19}}.
        These raise a fundamental question: \emph{can we design a model to give the maximal flexibility of defining dependencies on  neighbors?}
    }

  
    In this work, we propose a \textbf{P}artial \textbf{P}ermutation \textbf{I}nvariant \textbf{N}ode \textbf{E}mbedding method (\textbf{\OM}) by developing a new notion of \textit{partial permutation invariant set function}, that can
    \begin{itemize}[leftmargin=*]
        \item learn node representations via a \textit{universal} graph embedding function $f$, without pre-defining pairwise similarity, specifying random walk parameters, or choosing aggregation functions such as 
        element-wise mean, a max-pooling neural network, or 
        long-short term memory units (LSTMs);
        \item capture the arbitrary permutation invariant relationship of each node to its neighbors;
        \item be applied to both homogeneous and heterogeneous graphs with complicated types of nodes.
    \end{itemize}
  
    Evaluation results on benchmark data sets show that the proposed {\mc \OM} outperforms the state-of-the-art approaches on producing node vectors for classification tasks.  
    
{\mn    
{\noindent\bf Notations:} 
Throughout this paper, we use following notations
\begin{itemize}
\item $\gG = \{\gV, \gE \}$ denotes a graph with vertex set $\gV$ and edge set $\gE$. The corresponding lower case characters $v$ and $e$ represents a single vertex and edge.
\item $\hat{f}$ denotes an approximation to a function f.
\item $\mP$ denotes permutation matrix.
\item $S_N$ denotes a symmetric group.
\item $T_\pi$ denotes a permutation operator.
\item $\sigma(\cdot)$ represents a non-linear activation function.
\end{itemize}
}
   
\section{Related Work}\label{sec:relatedwork}

    The main difference among various graph embedding methods lies in how they define the ``closeness'' between two nodes~\cite{survey-CAI}.  First-order proximity, second-order proximity or even high-order proximity have been widely studied for capturing the structural relationship between nodes~\cite{tang2015line,yang2017fast,NE-Wasserstein-KDD18}. 
    Comprehensive reviews of graph embedding can be found in \cite{survey-CAI, survey-jure, survey-Goyal,yang2017fast}. 
    In this section, we discuss the relevant graph embedding approaches in terms of how node closeness is measured, to highlight our contributions on capturing neighborhood dependency in a most general manner. {\mn This section ends up with the review about set functions which is related to the technology we used in this paper.}
  
    \noindent{\bfseries Matrix Analysis on Graph Embedding:}
        As early as 2011, a spectral clustering method~\cite{tang2011leveraging} was proposed to take the eigenvalue decomposition of a normalized Laplacian matrix of a graph as an effective approach to obtain an embedding of nodes.
        Other similar approaches choose different similarity matrices (from the Laplacian matrix) to make a trade-off between modeling the ``first-order similarity'' and modelling ``higher-order similarity''~\cite{GraRep, HOPE, rossi2018higher}.
        Node content information can also be fused in the pairwise similarity measure, e.g., in text-associated DeepWalk (TADW)~\cite{yang2015network}, as well as node label information, which results in semi-supervised graph embedding methods, e.g., max-margin DeepWalk (MMDW)~\cite{tu2016max}. 
        Recently, an arbitrary-order proximity-preserving graph embedding method is introduced in \cite{Arbitrary-Order-KDD18} based on matrix eigen-decomposition, which is applied to a pre-defined high-order proximity matrix. {\rn Furthermore, \cite{liao2019lanczosnet} proposes Lanczos network with Lanczos algorithm to construct low rank approximations of the graph Laplacian for graph convolution.}
        For heterogeneous networks, \cite{huang2017label} propose a label-involving 
        matrix analysis to learn the classification result of each vertex within a semi-supervised framework.
        
    \noindent{\bfseries Random Walk on a Graph to Node Representation:}
        Both deepwalk~\cite{perozzi2014deepwalk} and node2vec~\cite{grover2016node2vec} are graph embedding methods to solve the node embedding problem. They convert the graph structures into a sequential context format with random walk~\cite{lovasz1993random}. 
        Thanks to the pioneering work of~\cite{mikolov2013distributed} for word representation learning of sentences, deepwalk inherits the learning framework for words representation learning in paragraphs to generate the representation of nodes in random walk context.
         Then node2vec evolves such the idea with additional hyper-parameters tuning for the trade-off between depth-first search (DFS) and width-first search (WFS) to control the direction of random walk. 
        Struc2vec~\cite{ribeiro2017struc2vec} also utilizes the multilayer graph to construct the node representations. 
        \cite{selfpaced-NE-KDD18} proposes a self-paced graph embedding by introducing a dynamic negative sampling method to select difficult negative context nodes in the training process. 
        Planetoid~\cite{yang2016revisiting} is 
        a semi-supervised learning framework by guiding random walk with available node label information. 
        The heterogeneity of graph nodes is often handled by a heterogeneous random walk procedure~\cite{dong2017metapath2vec}, or selected relation pairs~\cite{chang2015heterogeneous}. 
        \cite{tang2015pte} considers the predictive text embedding problem on a large-scale heterogeneous text network and the proposed method is also based on pre-defined heterogeneous random walks. 
  
  
    \noindent{\bfseries Neighborhood Encoders to Graph Embedding:}
        There are also methods focusing on aggregating or encoding the neighbors' information to generate node embeddings. 
        DNGR~\cite{DNGR} and SDNE~\cite{SDNE} introduce autoencoders to construct the similarity function between the neighborhood vectors and the embedding of the target node. DNGR defines neighborhood vectors based on random walks and SDNE introduces adjacency matrix and Laplacian eigenmaps to the definition of neighborhood vectors. 
        GraphWave~\cite{donnat2018learning} learns the representation of each node's neighborhood via leveraging heat wavelet diffusion patterns. 
        Although the idea of autoencoder is a great improvement, these methods are painful computationally expensive
        when the scale of the graph is up to millions of nodes. 
        As a result, neighborhood aggregation and convolutional encoders are employed to integrate local aggregation for node  embedding, such as GCN~\cite{kipf2016semi, kipf2016variational, schlichtkrull2018modeling, van2017graph}, FastGCN~\cite{fastGCN}, column networks~\cite{pham2017column}, the GraphSAGE algorithm~\cite{GraphSAGE}, GAT~\cite{velickovic2017graph}. A recent DRNE~\cite{NE-recursive-KDD18} method uses layer normalized LSTM to approximate the embedding of a target node by the aggregation of its neighbors' embeddings. {\rn And \cite{xu2018powerful} utilizes a set function as a universal approximator to distinguish different graphs with respect to 
        graph classification tasks.}
        The main idea of these methods is involving an iterative or recursive aggregation procedure, e.g., convolutional kernels or pooling procedures to generate the embedding vectors for all nodes, and such aggregation procedures are shared by all nodes in a graph. 
  
    The above-mentioned methods work differently on how they use neighboring nodes for node embedding.
    They require pre-defining pairwise similarity measure between nodes, specifying random walk parameters, or choosing aggregation functions. 
    In practice, it usually takes a lot of effort to tune these parameters or try different measures, especially when graphs are complicated with nodes of multiple types, i.e., heterogeneous graphs. 
    This work hence targets on making neighboring nodes play their roles in a most general manner such that their contributions are learned but not user-defined. 
    The resultant embedding method has the flexibility to work on any types of homogeneous and heterogeneous graph.   
  
    Our proposed method {\OM}  has a natural advantage on avoiding any manual manipulation of random walking strategies or designs for the relationships between different types of nodes. 

{\mn    
    \noindent{\bfseries Set functions:}    
    \cite{zaheer2017deep} introduces the notion of set functions as a universal approximator to measure the permutation invariant property of sets but only provides a less rigorous skeleton proof. A very recent work \cite{yarotsky2018universal} further improves the theoretical analysis on invariant maps by neural networks. The notion of partial permutation invariant set functions proposed in this paper is a more generic version of the set function. We find a neater form than \cite{yarotsky2018universal} even in the special case and also provide rigorous proofs for the representation theorem.
    }
  
\section{The Proposed {\OM} {\mn framework}}
    In this section, we first formally define the problem, and then introduce a new definition --- partial permutation invariant set function. {\rn \mn This section ends up with the proposed {\OM} framework whose key is the representation theorem of the partial permutation invariant set function.}
 
        We target on designing graph embedding models for general graphs that may include $K$ different types of nodes ($K$=1 corresponds to the homogeneous graphs and $K\geq 2$ corresponds to heterogeneous graphs). 
        Formally, a graph $\gG = \{\gV, \gE \}$, where the node set $\gV = \bigcup_{k=1}^{K}\gV_k$, i.e., $\gV$ is composed of $K$ disjoint types of nodes.  
        One instance of such a graph is the academic publication network, which includes different types of nodes for papers, publication venues, author names, author affiliations, research domains etc.  
        Given such a graph $\gG$, our goal is to learn the embedding vector  for each node in this graph. 
        
        We use $\bx^\nodev \in \sR^d$ to denote the representation of node $\nodev$. The node $\nodev\in\gV_k$ can be represented by its neighbors' embedding vectors via a function $\rvf$
        \begin{equation}\label{eq.embed}
            \bx^\nodev = \rvf(\mX^\nodev_1, \mX^\nodev_2, \cdots, \mX^\nodev_K),
        \end{equation}
where $\mX^\nodev_k$ is a matrix with column vectors corresponding to the embedding of node $\nodev$'s neighbors in type $k$. $\mX^\nodev_k$ could also be the {\mn representation} vectors associating with node $\nodev$'s type $k$ neighbors. We use $d$ to denote the dimensions of the embedding vector. Note that they way we have defined the function $f$ implies that it is node dependent. For learning to be possible, the embedding functions for different nodes will share common parameters, as will become clear in Section~\ref{sec.PINE}.

\subsection{\mn Partial permutation invariant set functions}
    {\mn An undirected} graph can be uniquely determined by defining the set of neighborhoods. {\rn \mn Therefore, the key to defining the graph embedding lies in \emph{how to model the dependence of each node to its neighbors}, that is, what function $\rvf$ in \eqref{eq.embed} to choose.}
        Most existing approaches 
        only (either explicitly or implicitly)  stress on some specific forms to characterize the dependence between each node and its neighbors while ignoring other potential dependence.

        We propose a universal graph embedding model that does not pre-define the dependence form between each node and its neighbors due to the key observation:
all neighboring nodes reachable from a target node $\nodev$
{\mn are not distinguishable from the view of the target node if they belong to the same type. To formally define the function satisfying this property, we introduce a new notation named \emph{partial permutation invariant set function}.
            \begin{definition} \label{def.partial_invaraint_permutation_matrix}
\about{Partial permutation invariant set function} 
            Given $W \coloneqq W_1\times W_2 \cdots \times W_K $ where $W_k \coloneqq \bbR^{{M_k}\times N_{k}}$, a continuous real valued map $\rvf:W\longrightarrow\bbR^d$ is partially permutation invariant if
            \begin{align}
                \rvf(\mX_1\mP_1,\mX_2\mP_2, \cdots, \mX_K\mP_K) = \rvf(\mX_1,\mX_2, \cdots, \mX_K)
                \end{align} 
                \label{eq.permutation_matrix_invariant}
            for all permutation matrices $\mP_k \in \bbR^{N_k\times N_k}$ and $k\in[K]$.
            \end{definition}   
This definition essentially requires the function value of $\rvf(\cdot)$ to be invariant to swapping any two columns of $\mX_k$.}        
\subsection{\OM: the representation of partial permutation invariant set function} \label{sec.PINE}
Unfortunately, this function is not simply learnable because the permutation property is hard to guarantee directly. One straightforward idea to represent the partial permutation invariant set function is to define it in the following form
                \begin{align}
                    \bx^{\nodev}& = \rvf(\mX^\nodev_1, \cdots, \mX^\nodev_K)\\
                    & \coloneqq \sum_{\mP_1\in \sP_{|\gV_1^\nodev|}} \sum_{\mP_2\in \sP_{|\gV_2^\nodev|}} \cdots \sum_{\mP_K\in \sP_{|\gV_K^\nodev|}}  \rvt(\mX^\nodev_1\mP_1, \cdots, \mX^\nodev_K\mP_K)
                    \label{eq:simple}
                \end{align} 
            where $\gV_k^\nodev$ denotes type-$k$ neighbors of node $\nodev$ and $\sP_{|\gV_k^\nodev|}$ denotes the set of $|\gV_k^\nodev| \times |\gV_k^\nodev|$ permutation matrices for any $k \in [K]$, $\mX^\nodev_k\mP_k$ is to permute the columns in $\mX^\nodev_k$, and $\rvt(\cdot)$ is a properly designed function.
            It is easy to verify that the function defined in \eqref{eq:simple} is partial permutation invariant, but it is {\rrc intractable} because it involves $\prod_{k=1}^N (|\gV_k^\nodev|\,!)$ ``sum'' items. 
            Our solution of learning function $\rvf$ is then based on the following important theorem, which gives a neat {\revisec and general} way to represent any partial permutation invariant set function.
            \begin{theorem}
            \label{theorem6} {\bf\mn [Representation theorem of partial permutation invariant set functions]}
                Let $\rvf$ be a continuous real-valued function defined on a compact set with the following form
                \small
                  \begin{align*}
                    \rvf(&\underbrace{\bx_{1,1}, \bx_{1,2}, \cdots, \bx_{1,N_1}}_{G_1}, 
                    \underbrace{\bx_{2,1}, \cdots, \bx_{2,N_2}}_{G_2}, \cdots, 
                    \underbrace{\bx_{K,1}, \cdots, \bx_{K,N_K}}_{G_K}),
                  \end{align*}
                \normalsize
                where $\bx_{k,n} \in \R^{M_k}$. If function $\rvf$ is partial permutation invariant, that is, any permutations of the elements within the group $G_k$ for any $k$ does not change the function value, then there must exist functions $\rvh(\cdot)$ and $\{\rvg_k(\cdot)\}_{k=1}^K$ to approximate $\rvf$ with arbitrary precision in the following form
                 \begin{align} 
                 \rvh\left(\sum_{n=1}^{N_1}\rvg_1 (\bx_{1,n}),\sum_{n=1}^{N_2}\rvg_2 (\bx_{2,n}),\cdots,\sum_{n=1}^{N_K}\rvg_K(\bx_{K,n})\right).
                 \label{eq:URT}
                  \end{align}
            \end{theorem}

    The rigorous proof is provided in Appendix~\ref{apx:theorem3-1}. This result suggests a neat but universal way to represent any partial permutation invariant set function. For instance, a popular permutation invariant set function widely used in deep learning $\max(\cdot)$ can be approximated with an arbitrary precision by 
      \[
        \max(x_1, x_2, \cdots, x_N) \approx \mathtt{h}\left( \sum_{i=1}^N \mathtt{g}(x_i) \right)
    \]
    with $\mathtt{g}(z) = [\exp(kz) \cdot z, \exp(kz)]$, and $\mathtt{h}([z, z']) = z/z'$, as long as $k$ is large enough. This is because 
        \begin{align*} 
          &\max(x_1, x_2, \cdots, x_N) =  \lim_{k\rightarrow \infty} \mathtt{h}\left( \sum_{i=1}^N \mathtt{g}(x_i) \right) \\
          & = \lim_{k\rightarrow \infty} \left({\sum_{i=1}^N \exp (kx_i)}\right)^{-1}\sum_{i=1}^N {\exp(kx_i)} \cdot x_i.
        \end{align*} 
Theorem \ref{theorem6} only establishes the existence of the approximation. To obtain concrete forms of $\rvh(\cdot)$ and $\rvg_k(\cdot)$'s, one can always use three layers neural networks to approximate it (to any precision) \cite{cybenko1989approximations, hornik1991approximation}, for example, $ \rvh(\bz) = {\bm\sigma}(\mB{\bm\sigma}(\mA\bz + \va)+\vb)$. Our following theorem shows that we can even choose simpler and neater form than three layers neural network for $\rvh(\cdot)$ and $\rvg_k(\cdot)$ to approximate an arbitrary $\rvf(\cdot)$ as a whole. More specifically, a two-layers neural network is enough. For simplicity, we consider the case that the image of $\rvh(\cdot)$ or $\rvf(\cdot)$ is one dimension. The case with a high dimensional image can be simply applied based on the one dimensional case.



\begin{theorem}\label{theorem:main}
The functions $\rvh(\cdot)$ and $\rvg_k(\cdot)$ in Theorem~\ref{theorem6} can be chosen in the following form (assuming that the image of $\rvh(\cdot)$ is one dimension):
\begin{align*}
&\rvh\left([\bz_1^\top,\cdots, \bz_K^\top]^\top=: \bar{\bz}~|~\vc, \mW\right) =  \vc^\top {\bm\sigma}(\mW\bar{\bz})
\\
&\rvg_k\left(\bx~|~T, \{\vu_t, \va_t\}_{t=1}^{T}\right) =   
[
{\bm\sigma}((\vu_1 \otimes \va_{1}) \bx  + \vv)^\top, \cdots,
\\
& \qquad \qquad \qquad \qquad \qquad \quad{\bm\sigma}((\vu_{T}\otimes \va_{T}) \bx+ \vv)^\top]^\top,
\end{align*}
where ${\bm\sigma}(\cdot)$ is the element-wise squashing activation function, and we omit the subscript $k$ all hyper-parameters of $\rvg_k(\cdot)$.
\end{theorem}

Based on \thref{theorem:main}, we will use neural networks with appropriate structure as specified by the theorem to approximate the embedding function for a node $\nodev$ in \eqref{eq.embed}. In particular, for $\hat{f}(\cdot)$ with one dimension image, it can be cast into the following form:
\small
 \begin{equation}\label{eq.nn-main}
            \sum_{l=1}^L \evc_l {\sigma}\left(\sum_{k=1}^K\sum_{t_k=1}^{T_k}\sum_{q_k=1}^{Q_k} \evw^{(l,k)}_{t_k q_k} \sum_{n=1}^{N_k}\sigma\Big(\evu^{(k)}_{q_k}{\va^{(k)}_{t_k}}^{\top}\bx_{k,n}+ \evv^{(k)}_{q_k}\Big) + \evb_l\right)
        \end{equation}
        \normalsize
        with large enough hyper parameter $T_k$'s and $Q_k$'s and proper real value coefficients $\evc_l, \evb_l, \evw^{(l,k)}_{t_k,q_k}, \evu^{(k)}_{q_k}, \evv^{(k)}_{q_k}$ and $\va^{(k)}_{t_k}$. 

Generally, for an arbitrary node $\nodev\in\gV_k, k\in[K]$, we need to specify a function $\hat{f}$ to aggregate the neighborhood information. However, it will be overfitting for a dataset if we define totally different $\hat{f}$ for each $\nodev$. Hence, we define a function $\hat{f}$ for a node $\nodev$ via only varying the summation over the neighbor size $\left(\sum_{n=1}^{N_k^\nodev}\right)$ with respect to different $N_k^\nodev$, but reusing other parameters such as $T_k, Q_k, \evc_l, \evb_l, \evw^{(l,k)}_{t_k,q_k}, \evu^{(k)}_{q_k},$ and $ \evv^{(k)}_{q_k}$ across the entire graph.

The objective function for the neural network parameter optimization will depend on application. As one example, a two-norm cost function of the embedded vectors $\bx^\nodev$ and the embedding output as provided by the neural network can be minimized for consistency. This is a joint optimization problem: both the embedding vectors $\bx^\nodev$'s and the embedding functions are jointly optimized. The numerical optimization algorithm and complexity are similar to those for standard deep neural networks. In a semi-supervised setting, it is also possible to incorporate
a supervised component into the objective function; see numerical examples in \secref{sec.exp}, and specifically \eqref{eq:homo-optimization}.
    
    \section{Empirical Experiments Study}\label{sec.exp}
    {\bn 
    In the section, we validate and report the performance of the proposed partial permutation invariant set function theorem on various aspects of graph node embedding learning tasks comparing to state-of-the-art algorithms: (1) to evaluate the applicability of {\OM} on embedding problem of general graphs, we conduct experiments on both homogeneous and heterogeneous graphs; and {\mn (2) we also visualize the embedding vectors obtained by {\OM} to 2D space via t-SNE with respect to the true and predicted labels; (3) for ablatoin study, we investigate the impact of hyper-parameters for {\OM} and show the performace on two datasets, Cora and Wikipedia. More experimental results of ablation study are provided in Appendix B. }
    
    \subsection{Evaluation on Homogeneous Graphs}
      First we consider the multi-class node classification problem over the homogeneous graphs. 
      Given a graph with partially labeled nodes, the goal is to learn the representation for each node for predicting the class for unlabeled nodes. To fulfill the learning requirements, we have a basic assumption that the embedding of an arbitrary node in the graph can be calculated via {\OM} with the neighborhood as the input. For short, we denote the {\OM} embedding function to aggregate neighborhood information for a node $\nodev$ by $\bx^\nodev = \left[\hat{f}_1(\gX^\nodev), \hat{f}_2(\gX^\nodev), \cdots, \hat{f}_d(\gX^\nodev)\right]^\top \in \R^d$, where $\gX^\nodev\coloneqq(\mX^\nodev_1, \mX^\nodev_2, \cdots, \mX^\nodev_K)$ contains all embeddings of neighbors of $\nodev$.
      Since this section evaluates the homogeneous graphs, the $K$ in all $\hat{\rvf}(\cdot)$s is set to 1.
      To fulfill the requirement of a specific learning task, we propose an overall learning model with {\OM} by involving an unsupervised component and a supervised component at the same time
      \begin{align}
        \min_{\substack{\{\bx^\nodev\}_{\nodev\in\gV},\\ \{\hat{f}_m\}_{m=1}^d, \theta}} \frac{1}{\lambda|\gV|}\sum_{\nodev\in\gV}\left\|\bx^\nodev - \left[\hat{f}_1(\gX^\nodev),\cdots, \hat{f}_d(\gX^\nodev)\right]^\top\right\|^2 
         + \frac{1}{|\gV_{\textrm{label}}|}\sum_{\nodev\in\gV_{\textrm{label}}}\ell_\theta(\bx^\nodev, \vy^\nodev)&.\label{eq:homo-optimization}
      \end{align}
      The first term of the objective in \eqref{eq:homo-optimization} is the unsupervised learning component, which restricts the representation error between the target node and its neighbors with $L_2$ norm since it is allowed to have noise in a practical graph. 
      The second term 
      is the supervised component, which is flexible to be replaced with any designed learning task on the nodes in a graph. 
      For example, to a regression problem, a least square loss can be chosen to replace $\ell_\theta (\cdot )$ and a cross entropy loss can be used to formulate a classification   problem. 
      
       The details of {\OM} for multi-class case are as follows: 
        \textbf{(1) Supervised Component:} \textit{Softmax} function is chosen to formulate our supervised component in \eqref{eq:homo-optimization}. For an arbitrary embedding $\bx\in\R^d$, we have the probability term as $\text{P}(y=i|\bx) = \frac{\exp(\vw_i^\top \bx + b_i)}{\sum_{j=1}^C \exp(\vw_j^\top \bx + b_j)}$ for predicting $\bx$ with class $i$, where $\vw_i\in\mathbb{R}^{d}$ and $b_i\in\mathbb{R}$ are classifier parameters for class $i$, and $C$ is the number of classes.  
          Therefore, the supervised component in \eqref{eq:homo-optimization} is formulated as $ \frac{1}{|\gV_{\text{label}}|} \sum_{\nodev\in\mathcal{V}_{\textrm{label}}}\sum_{i=1}^C [ -y^\nodev_i \log \text{P}(y^\nodev=i | \bx^\nodev)] + \lambda_w\sum_{i=1}^C \textrm{Reg}(\vw_i) $, where $y^\nodev_i \in \{0, 1\}$ is the true label for training, $\textrm{Reg}(\vw_i)$ is an $L_2$ regularization for $\vw_i$, and $\lambda_w$ is chosen to be $10^{-3}$; 
        \textbf{(2) Unsupervised embedding mapping Component:} The balance hyper-parameter $\lambda_1$ is set to be $0.005$. We follow the formulation in \eqref{eq.nn-main} and $L=16$, $T_1=32$, and $Q_1=16$.
      We apply an ADAM algorithm to compute the effective solutions for the learning variables simultaneously.

      \begin{table*}[!t]
    \caption{Summary of Datasets} 
    \label{tab:models}
    \begin{center}
        {\fontsize{8}{8}\selectfont
        \begin{tabular}{cccccc|cc}
            \toprule
            & \textbf{Cora} & \textbf{Citeseer} & \textbf{Pubmed} & \textbf{Wikipedia} & \textbf{Email-eu} & \textbf{DBLP} & \textbf{BlogCatalog} \\ 
            \midrule
            \#Node & 2,708 & 3,312 & 19,717 & 2,405 & 1,005 &  $27K$ + $3.7K$ & 55,814 + 5,413\\
            \#Edge & 5,429 & 4,732 & 88,651 & 17,981 & 25,571 & 338,210 + 66,832 &  1.4M + 619K + 343K\\ 
            \midrule
            \#Classes & 7 & 6 & 3 & 17 & 42 & 4 (multi-label) & 5 (multi-label)\\
            \bottomrule
        \end{tabular}}
    \end{center}
\end{table*}

  {\mn
     \noindent{\bf Datasets:} We evaluate the performance of PINE and other methods for  comparison on five benchmark datasets: \textit{\textbf{Cora}}~\cite{mccallum2000automating}, \textit{\textbf{Citeseer}}~\cite{giles1998citeseer}, \textit{\textbf{Pubmed}}~\cite{sen2008collective}, \textit{\textbf{Wikipedia}}~\cite{sen2008collective}, and \textit{\textbf{Email-eu}}~\cite{leskovec2007graph}. The details of these five datasets are presented in Table~\ref{tab:models}.
  }
     
    {\mn \noindent {\bf Baseline methods:} To evaluate the learning capability of {\OM}, we compare it with baseline algorithms listed below:
          \begin{itemize}
            \item \textbf{Deepwalk}~\cite{perozzi2014deepwalk} is an unsupervised graph embedding method which relies on the random walk and word2vec method.  For each vertex, we  take 80 random walks with length 40, and set window size as 10. Since deepwalk is {\bf unsupervised}, we apply a logistic regression on the generated embeddings for node classification. \item \textbf{Node2vec}~\cite{grover2016node2vec} is an improved graph embedding method based on deepwalk. We set the window size as 10, the walk length as 80 and the number of walks for each node is set to 100. Similarly, the node2vec is {\bf unsupervised} as well. We apply the same evaluation procedure on the embeddings of node2vec as what we did for deepwalk.
            \item \textbf{Struc2vec}~\cite{ribeiro2017struc2vec} chooses the window size as 10, the walking length as 80, the number of walks from each node as 10, and 5 iterations in total for SGD.
            \item \textbf{GraphWave}~\cite{donnat2018learning} chooses the heat coefficient as 1000, the number of characteristic functions as 50, the number of Chebyshev approximations as 100, and the number of steps as 20.
            \item \textbf{WYS (Watch-Your-Step)}~\cite{abu2017watch} chooses the learning rate as 0.2, the highest power of normalized adjacency matrix as 5, the regularization coefficient as 0.1, and uses the ``Log Graph Likelihood'' as objective function.
            \item \textbf{MMDW}~\cite{tu2016max} is a {\bf semi-supervised} learning framework of graph embedding which combines matrix decomposition and SVM classification. We tune the method multiple times and take 0.01 as the hyper-parameter $\eta$ in the method which is recommended by the authors.
            \item \textbf{Planetoid}~\cite{yang2016revisiting} is a {\bf semi-supervised} learning framework. We set the batch size as 200, learning rate as 0.01, the batch size for label context loss as 200, and mute the node attributes as input while using softmax for the model output.
            \item \textbf{GCN (Graph Convolutional Networks)}~\cite{kipf2016semi} chooses the convolutional neural networks into the {\bf semi-supervised} embedding learning of graph.
            We eliminate the node attributes for fairness as well.
            \item \textbf{GATs (Graph Attention Networks)}~\cite{velickovic2017graph} choose the learning rate as 0.005, the coefficient of the regularization as 0.0005, and the number of hidden units as 64. To make the comparison fair, we mute the node attributes in the training of GATs as well.
          \end{itemize} }
     

    \noindent{\bf Experiment setup and results.}
        For a fair comparison, the dimension of representation vectors is chosen to be the same for all algorithms (the dimension is $64$). The hyper-parameters are fine-tuned for all of them. More experiment environment and comparison details are presented in the Appendix. 
       
        In this  multi-class classification scenario, we use \textit{Accuracy} as the evaluation criterion. The percentage of labeled nodes is chosen from $10\%$ to $90\%$ and the remaining nodes are used for evaluation. 
        All experiments are repeated for five times and we report the mean and standard deviation of the performance of each graph embedding method in Figure~\ref{fig:cora-citeseer-pubmed-wiki-email}.
        We can observe that in most cases, {\OM} outperforms other methods and in few cases, {\OM} performs the second best behind of MMDW.

        \begin{figure*}[!ht]
          \centering
              \centering
              \includegraphics[width=1\linewidth]{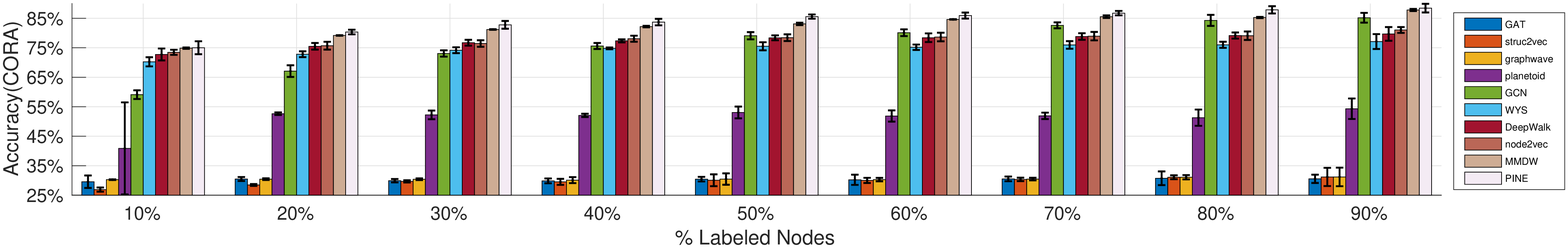}
      \\        \includegraphics[width=1\linewidth]{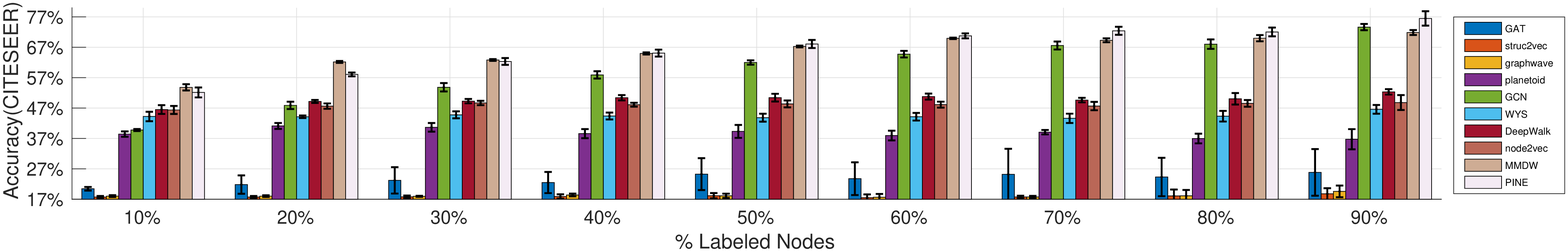}
      \\        \includegraphics[width=1\linewidth]{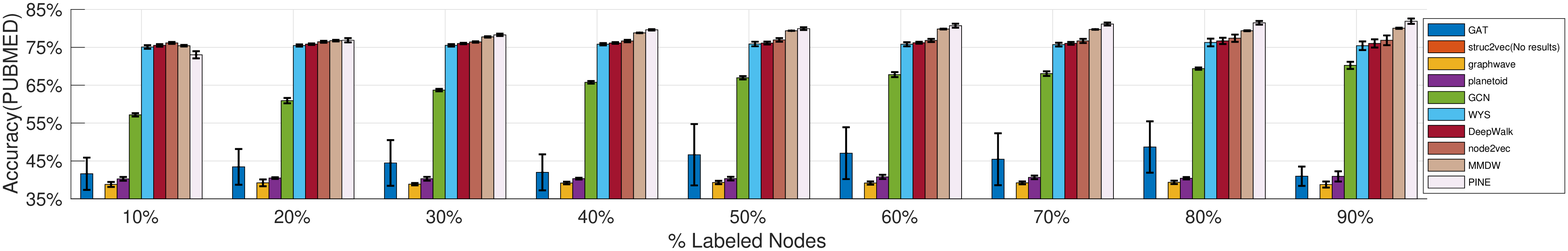}
      \\        \includegraphics[width=1\linewidth]{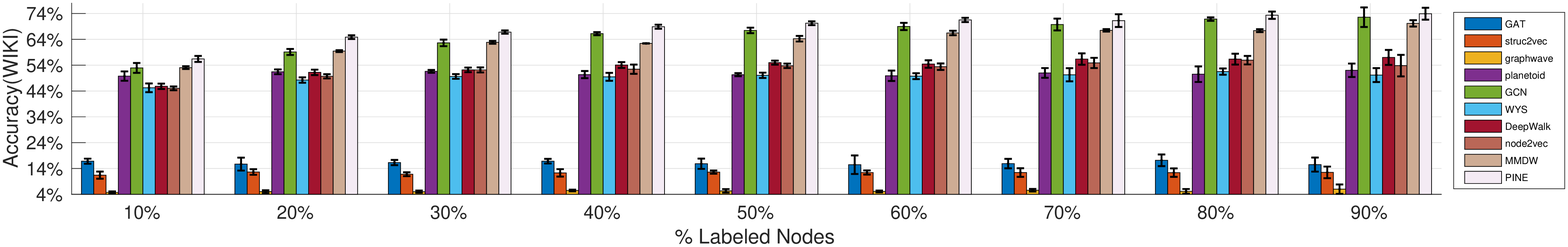}
      \\        \includegraphics[width=1\linewidth]{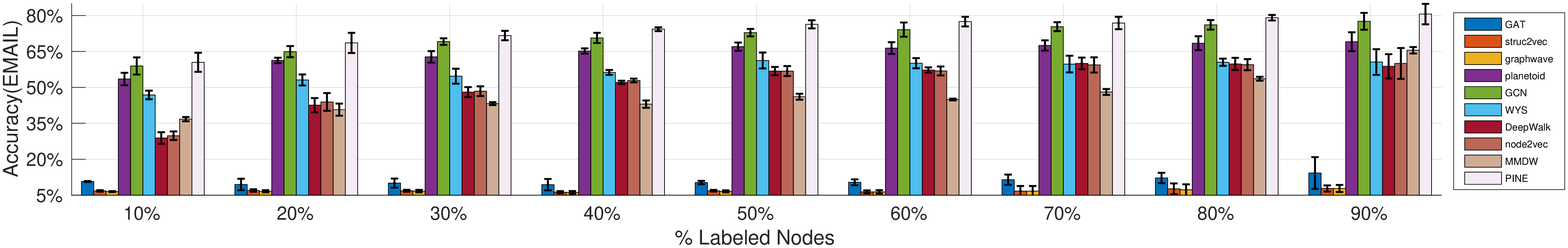}
          \caption{Accuracy (\%) of multi-class classification in Cora, Citeseer, Pubmed, Wikipedia, and Email-eu dataset
          }
          \label{fig:cora-citeseer-pubmed-wiki-email}
      \end{figure*}
      {\mn
      In addition to the reported accuracy under the graph node classification task. We also present the visualization of graph node embedding in the 2D space with the t-SNE method. The results of Cora, Citeseer, Email-eu, Pubmed, and Wikipedia are presented in Fig~\ref{fig:tsne_cora}. We
      present all the figures with the ratio of unlabeled nodes as 50\% for the
      node classification task. For each figure, we illustrate the t-SNE results with respect to the true labels and the predicted ones on test set. As what we expected, we can observe that the embedding vectors can be easily clustered with t-SNE for both true and predicted labels. There is the difference between the true label figure and predicted label one due to the difference of true and predicted labels.

       \begin{figure*}[ht]
         \begin{center}
         \begin{subfigure}[b]{.32\linewidth}
             \includegraphics[width=1\linewidth]{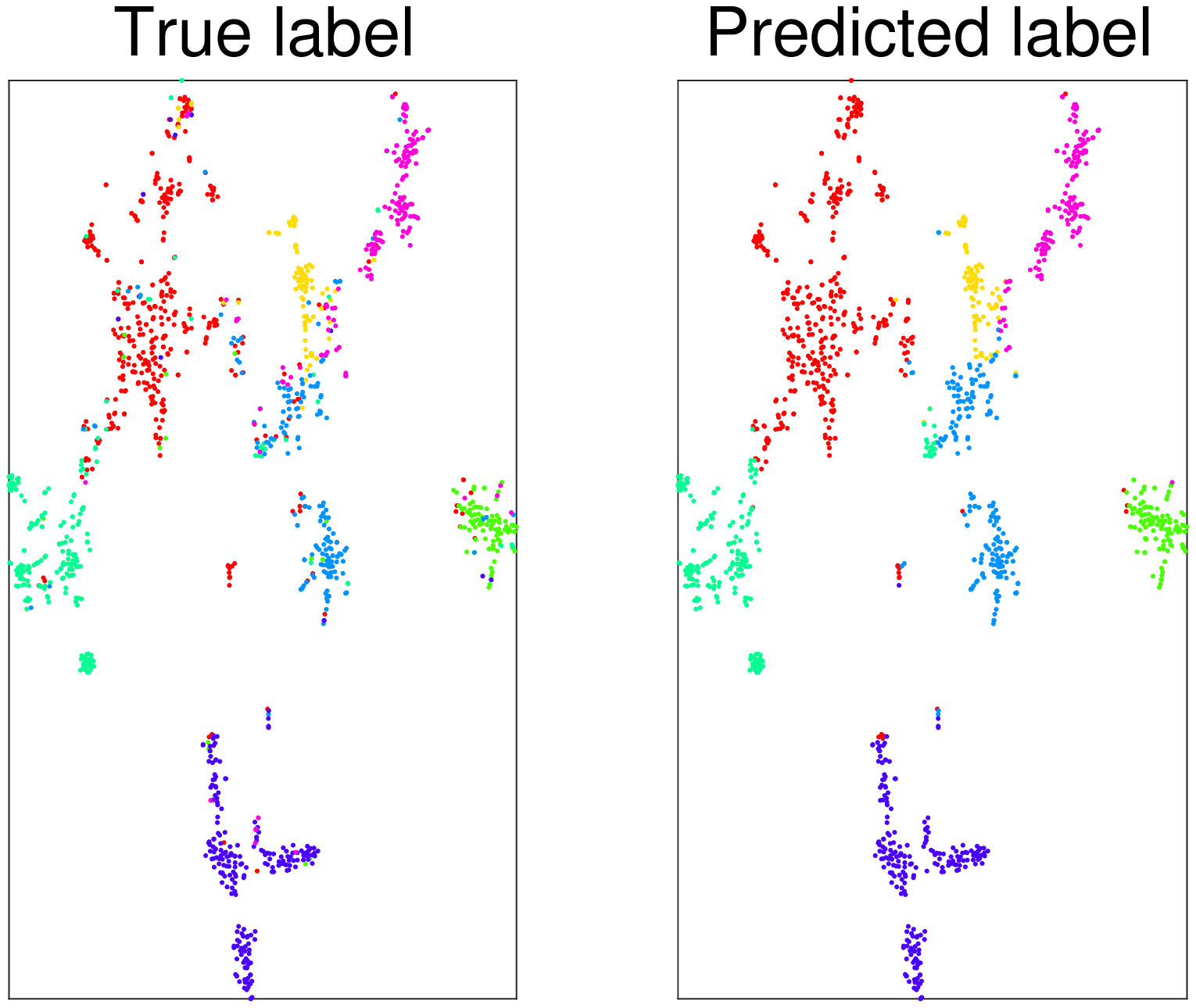}
                     \subcaption{t-SNE of Cora}
         \end{subfigure}
         \begin{subfigure}[b]{.32\linewidth}
             \includegraphics[width=1\linewidth]{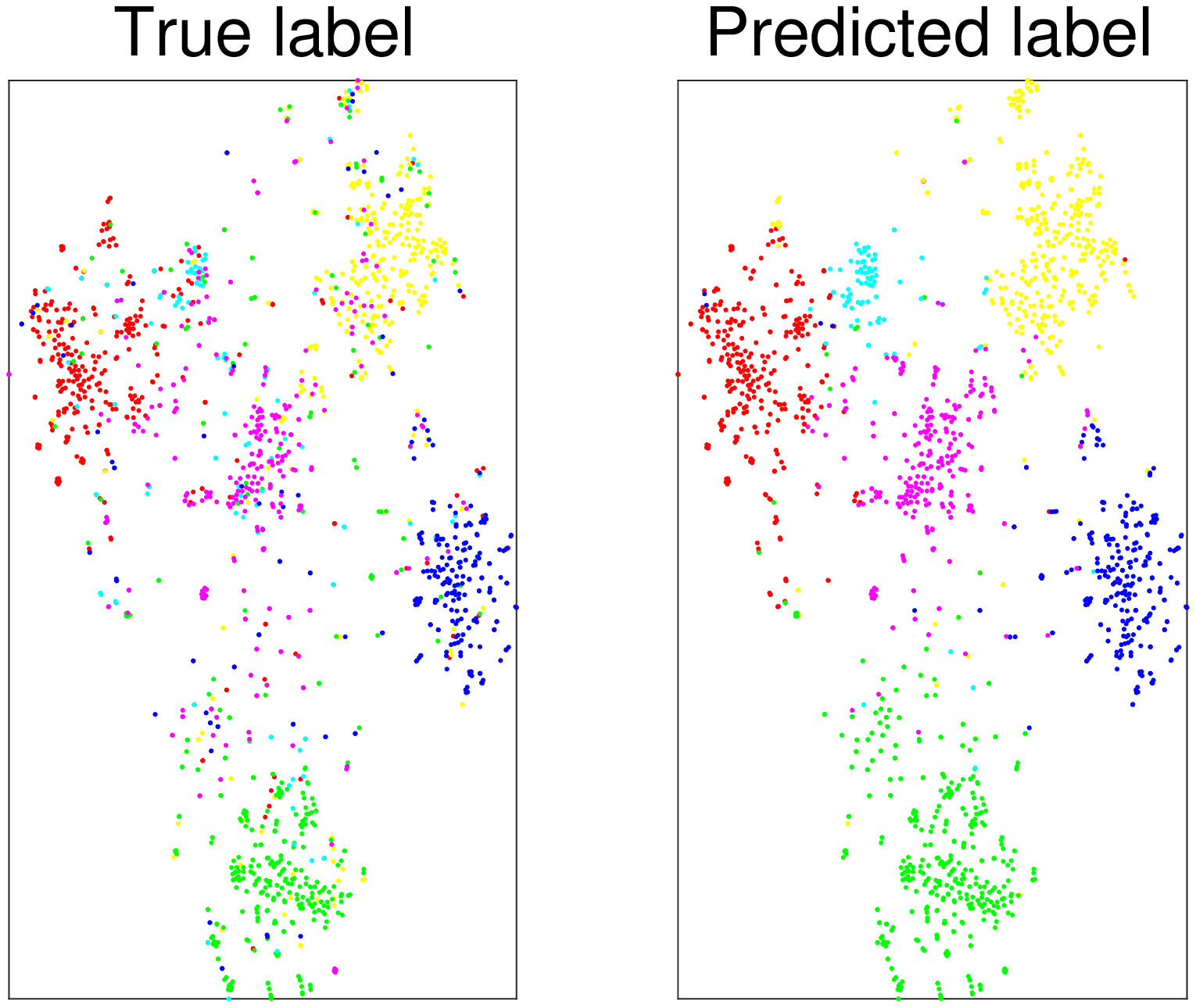}
                     \subcaption{t-SNE of Citeseer}
         \end{subfigure}
         \begin{subfigure}[b]{.32\linewidth}
           \includegraphics[width=1\linewidth]{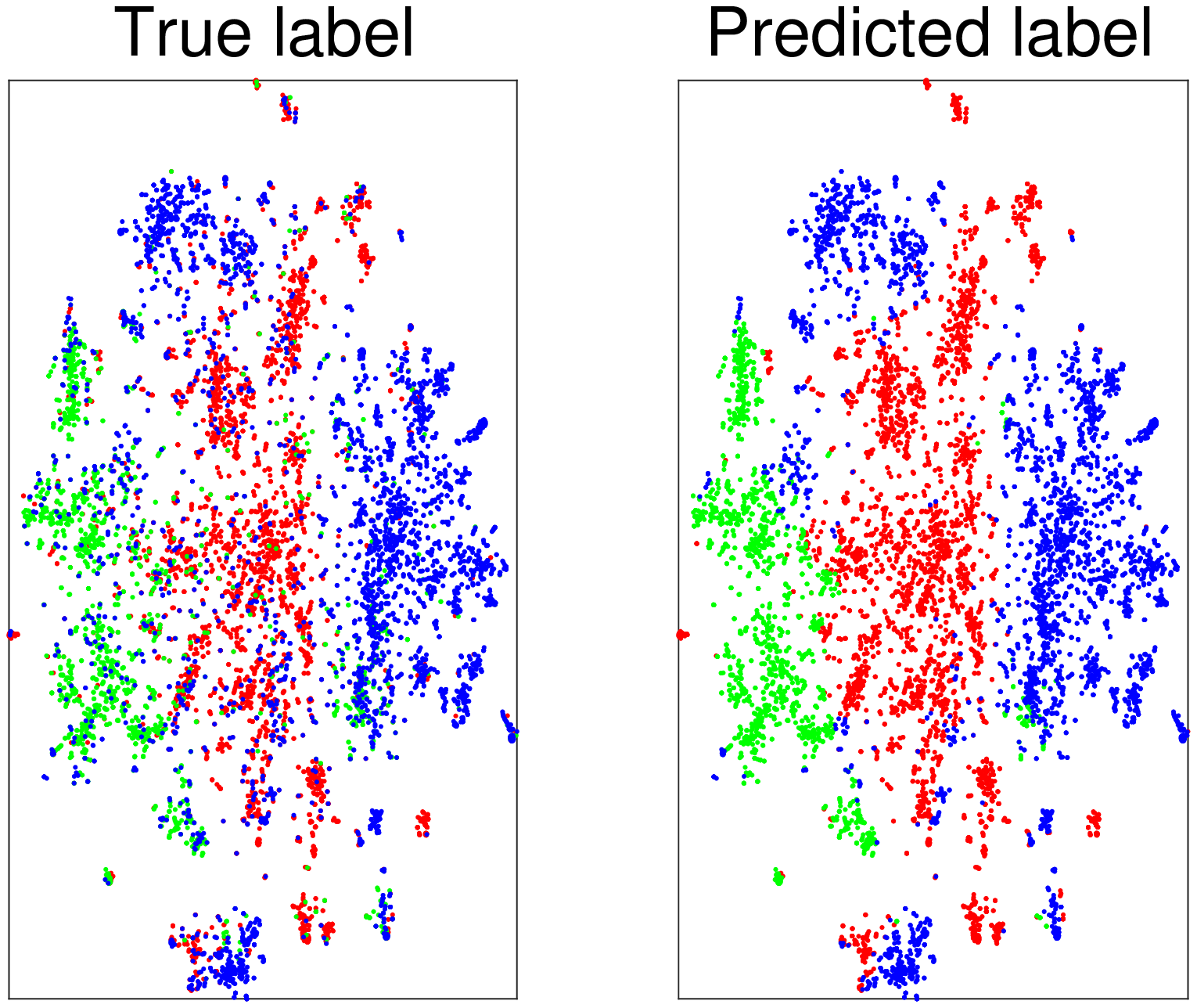}
                   \subcaption{t-SNE of Pubmed}
         \end{subfigure}
         \begin{subfigure}[b]{.32\linewidth}
           \includegraphics[width=1\linewidth]{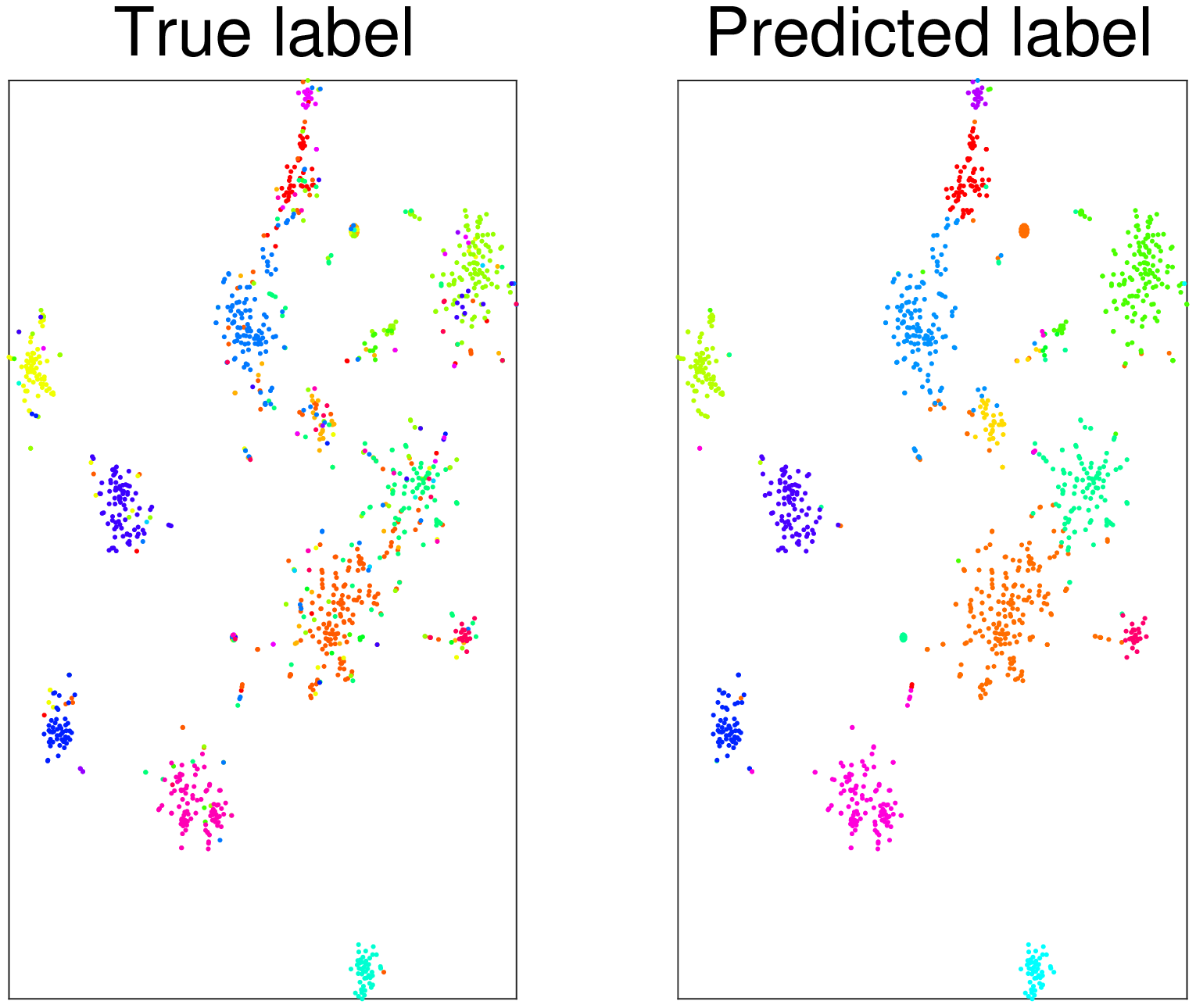}
                   \subcaption{t-SNE of Wikipedia}
         \end{subfigure}
         \begin{subfigure}[b]{.32\linewidth}
           \includegraphics[width=1\linewidth]{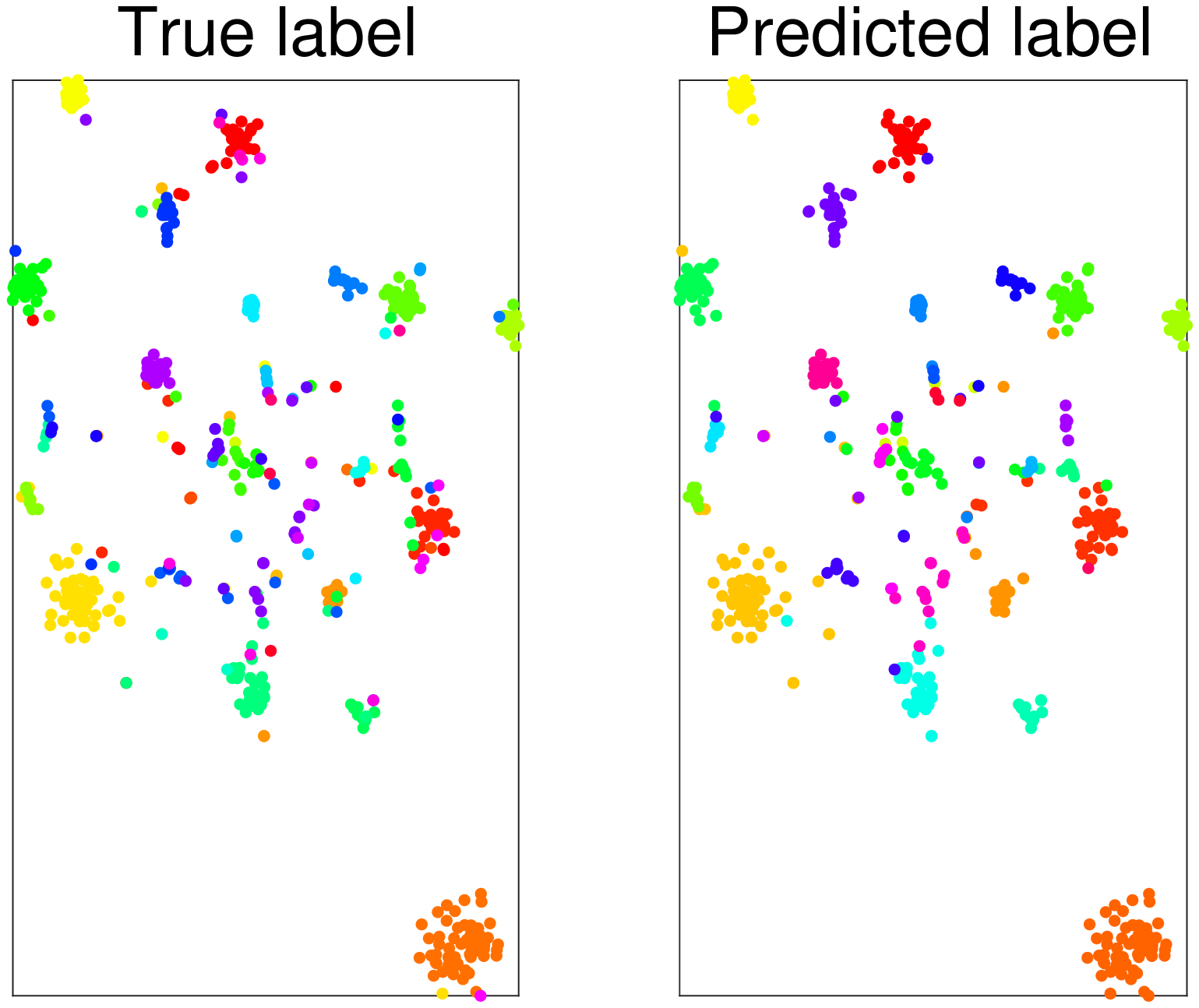}
                   \subcaption{t-SNE of Email-eu}
         \end{subfigure}
         \caption{The t-SNE of Cora, Citeseer, Pubmed, Wikipedia, and Email-eu dataset.}
         \label{fig:tsne_cora}
         \end{center}
         \end{figure*}

        Then, we provide several figures to show the hyper-parameter sensitivity for {\OM}. We take a case study on the sensitivity to the embedding dimensions on Cora and Wikipedia dataset. The results are shown in Figure~\ref{fig:cora_dim}. For dimensions of 8, 16, 32, and 64, we run the experiments for 5 times with the ratio of unlabeled nodes as 50\% for the node classification task and compute the mean and standard deviation of the accuracy results for GCN, MMDW, and {\OM}. As shown in the illustration, the performance rises along the increasing of the dimension of the representations for nodes for GCN, MMDW, and {\OM}, and {\OM} always achieves higher performance.
       \begin{figure*}[ht]
       \begin{center}
       \includegraphics[width=0.49\linewidth]{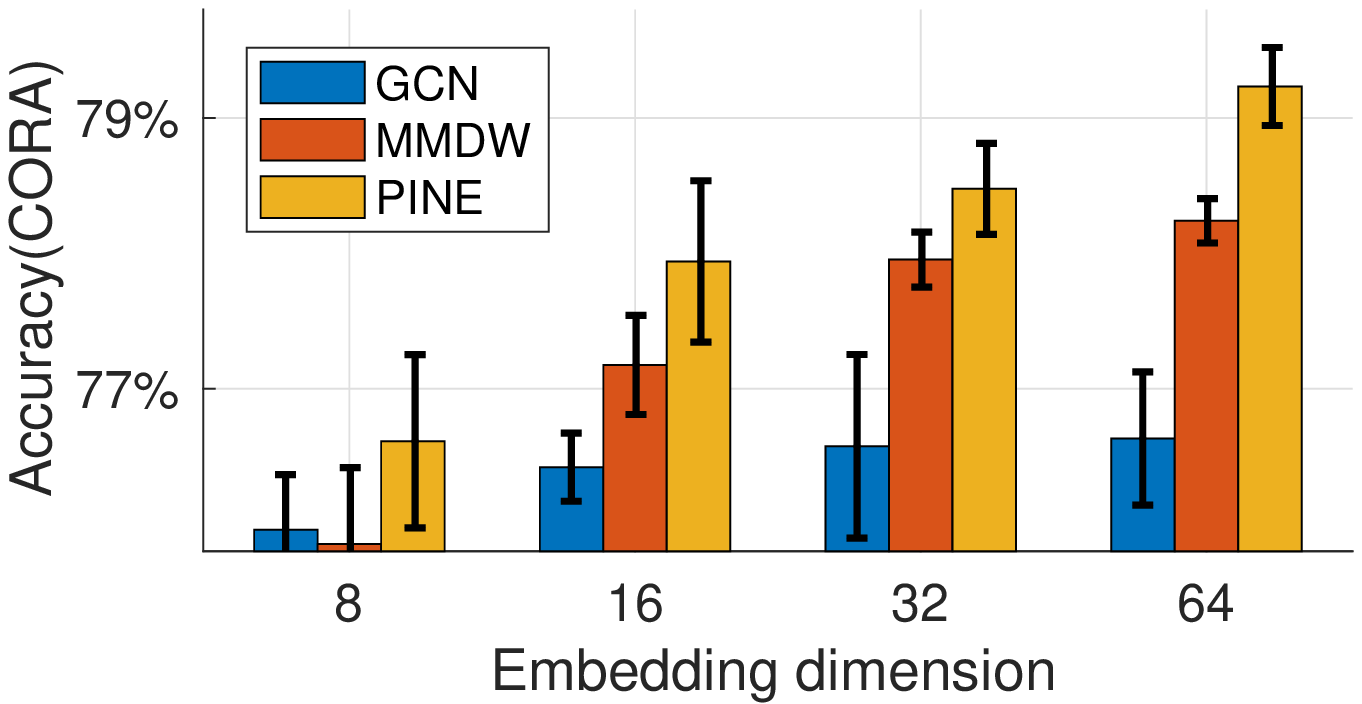}
       \includegraphics[width=0.49\linewidth]{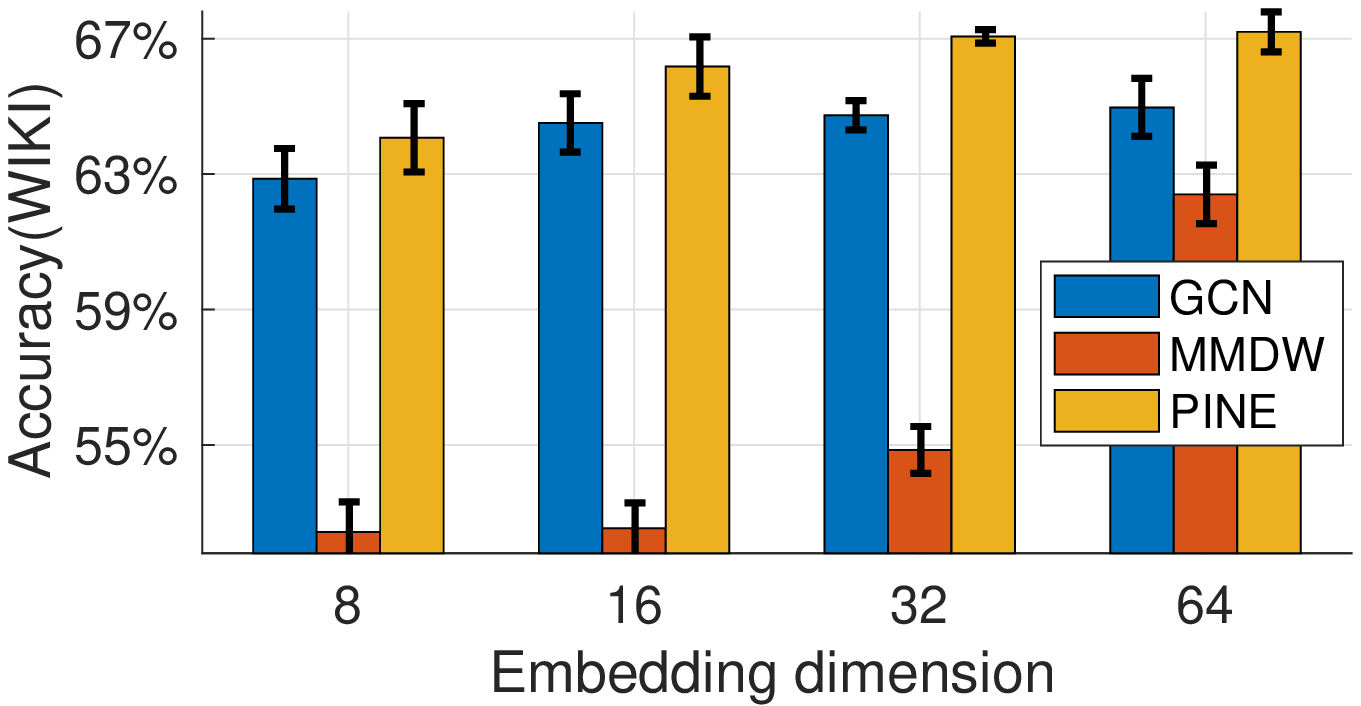}
       \caption{Dimension sensitivity illustration for {\OM} on Cora and
       Wikipedia.}
       \label{fig:cora_dim}
       \end{center}
       \end{figure*}

{\mn   \noindent{\bf Exploration under graph neural network framework}
In addition to the experiments on the comparison of the graph node classification task with our overall model in \eqref{eq:homo-optimization}, we also explore the potential of partial permutation invariant function in the existing Graph Neural Network (GNN). With the study of existing GNNs, we find that they rely on some specific aggregation function to measure the relationship between each target node and its neighborhood. It reminds us of substituting the neighborhood aggregation component with {\OM} in existing graph neural network (GNN) e.g., \textit{\textbf{GraphSAGE}}~\cite{GraphSAGE}  and \textit{\textbf{GAT}}~\cite{velickovic2017graph}, and evaluate its capability on neighborhood aggregation. It is expected that {\OM} under the GNN framework performs better than the orignal   \textit{\textbf{GraphSAGE}}~\cite{GraphSAGE}  and \textit{\textbf{GAT}}~\cite{velickovic2017graph}. 
(1) Dataset \textit{\textbf{PPI}} is a set of protein-protein interaction graphs, which poses an \textbf{inductive learning} problem. We take this chance to  validate   {\OM} by learning from 20 graphs with another 2 graphs for validation, and then classifying nodes in 2 other different graphs into 121 classes. We compare {\OM} with \textit{\textbf{GraphSAGE}}    and  \textit{\textbf{GAT}}  on this task. The results in Table~\ref{tab:PPI} show that  {\OM} has higher classification accuracy on the nodes in the previously unseen graphs. That justifies the superior performance of {\OM} in \textbf{inductive} setting.
(2) \textit{\textbf{Reddit}}~\cite{GraphSAGE}   is a large dataset including 232,965 nodes. We validate   {\OM} on this dataset to see its generalizability on learning from old data to predict new data. We use the first 20 days' data as the training set and the rest split up into the validation (30\%) and test set (70\%), which should be classified with multi-labels by choosing as accurate as possible from 50 labels.
Table~\ref{tab:Reddit} shows that  {\OM} has the best performance on this challenging task. 
Overall, we can conclude that 
{\OM} can be a suitable aggregator in GNN framework without subtle aggregation structure design.

\begin{table*}[t]
\caption{Inductive prediction results for the PPI dataset (micro-averaged F1 scores)} 
\label{tab:PPI}
\begin{center}
\begin{small}
\begin{sc}
\begin{tabular}{c|cc|ccc}
\hline 
\textbf{Methods} & \textbf{GraphSAGE-pool}~\cite{GraphSAGE} & \textbf{PINE}& \textbf{GraphSage}~\cite{velickovic2017graph}& \textbf{GAT}~\cite{velickovic2017graph}& \textbf{PINE} \\
\hline
Results&0.600&0.637&0.768&0.973& {\bf 0.985}\\
\hline
\end{tabular}
\end{sc}
\end{small}
\end{center}
\vskip -0.1in
\end{table*}

\begin{table*}[t]
\caption{Prediction results for the Reddit dataset (micro-averaged F1 scores)} 
\label{tab:Reddit}
\begin{center}
\begin{small}
\begin{sc}
\begin{tabular}{c|cccc}
\hline 
\textbf{Methods} & \textbf{Deepwalk} & \textbf{Deepwalk+features} & \textbf{GraphSAGE-pool}~\cite{GraphSAGE} & \textbf{PINE} \\
\hline
Results&0.324&0.691&0.949& {\bf 0.951}\\
\hline
\end{tabular}
\end{sc}
\end{small}
\end{center}
\end{table*}

}
    }


    \subsection{Comparison on heterogeneous graphs}
    We next conduct evaluation on heterogeneous graphs, where the learned node embedding vectors are used for multi-label classification. 
    Since multiple types of nodes are presented in heterogeneous graphs, we substitute the unsupervised embedding mapping component with 
    $$\sum_{k=1}^K \frac{1}{\lambda_k|\gV_k|}\sum_{\nodev\in \gV_k} \left\|\bx^\nodev - \left[\hat{f}_1(\gX^\nodev),\cdots, \hat{f}_d(\gX^\nodev)\right]^\top   \right\|^2.$$
   The supervised component in a  multi-label setting can be addressed by formulating a set of binary classification problem (one for each label).
    Therefore, \textbf{(1) Supervised Component:} 
        we apply logistic regression for each instance $x$ and its $i$-th label $y_i$ via letting 
            $\bar{\ell}(\bx, y_i) = \log(1 + \exp(\vw_i^\top \bx + b_i)) - y_i(\vw_i^\top \bx + b_i)$, 
            where $\vw_i\in\mathbb{R}^d$ and $b_i\in\mathbb{R}$ are classifier parameters for the $i$-th label. Then, defining $y^\nodev_i \in \{0,1\}$ to be the true label for training, the supervised component in \eqref{eq:homo-optimization} is formulated as $\frac{1}{|\gV_{\text{label}}|}\sum_{\nodev\in \gV_{\text{label}}} \sum_{i=1}^C \bar{\ell}(\bx^\nodev, y^\nodev_i) + \lambda_w \sum_{i=1}^C \textrm{Reg}(\vw_i)$, where $C$ is the number of labels, $\textrm{Reg}(\vw_i)$ is the regularization term for $\vw_i$, and $\lambda_w$ is chosen as $10^{-4}$; \textbf{(2) Unsupervised Embedding Mapping Component:} The balance hyper-parameter $[\lambda_1, \lambda_2]$ is set to be [0.2, 200]. And the hyper-parameter $[L, T_1, Q_1]$ in \eqref{eq.nn-main} is set to be $[8, 16, 8]$.
   
    {\mn
   \noindent{\bf Datasets:} The applied datasets include: 
   \textit{\textbf{DBLP}}~\cite{ji2010graph} is an academic community network. Here we obtain a subset of the large network with two types of nodes, authors and key words from authors' publications. The generated subgraph includes $27K$ (authors) + $3.7K$ (key words) vertexes. The link between a pair of author indicates the coauthor relationships, and the link between an author and a word means the word belongs to at least one publication of this author. There are 66,832 edges between pairs of authors and 338,210 edges between authors and words. Each node can have multiple labels out of four.
   \textit{\textbf{BlogCatalog}}~\cite{Wang-etal10} is a social media network with 55,814 users and according to the interests of users, they are classified into multiple overlapped groups. We take the five largest groups to evaluate the performance of methods. Users and tags are two types of nodes. The 5,413 tags are generated by users with their blogs as keywords. Therefore, tags are shared with different users and also have connections since some tags are generated from the same blogs. The number of edges between users, between tags and between users and tags are about 1.4M, 619K and 343K,  respectively. Each user is associated with multiple labels out of five. 
    The total number of labels is five and due to the multilabel classification setting, each user may have several possible labels.
    }
    
   \noindent{\bf Baseline Methods:}
    To illustrate the valid performance of {\OM} on heterogeneous graphs, we conduct the experiments on two stages: 
    (1) comparing {\OM} with Deepwalk~\cite{perozzi2014deepwalk} and node2vec~\cite{grover2016node2vec} on the graphs by treating all nodes as the same type ({\OM} with $K=1$ in a homogeneous setting);
    (2) comparing {\OM} with the state-of-the-art heterogeneous graph embedding method, \textit{metapath2vec}~\cite{dong2017metapath2vec}, in a heterogeneous setting. The hyper-parameters of the method are fine-tuned and \textit{metapath2vec++} is chosen as the option for the comparison. 
    
   \noindent{\bf Experiment Setup and Results:}
        For the datasets DBLP and BlogCatalog, we conduct the experiments on each of them and compare the performance among all methods mentioned above. 
        Since it is a multi-label classification task, we take \textit{F1-score (macro, micro)} as the evaluation metrics for the comparison. The users in BlogCatalog or authors in DBLP work are classification targets.  
        We vary the ratio of labeled nodes from 10\% to 90\%,   repeat all experiments for five times and report the mean and standard deviation of their performance in the Figure~\ref{fig:multilabel}. 

        \begin{figure*}[!t]
          \centering
        \includegraphics[width=1\linewidth]{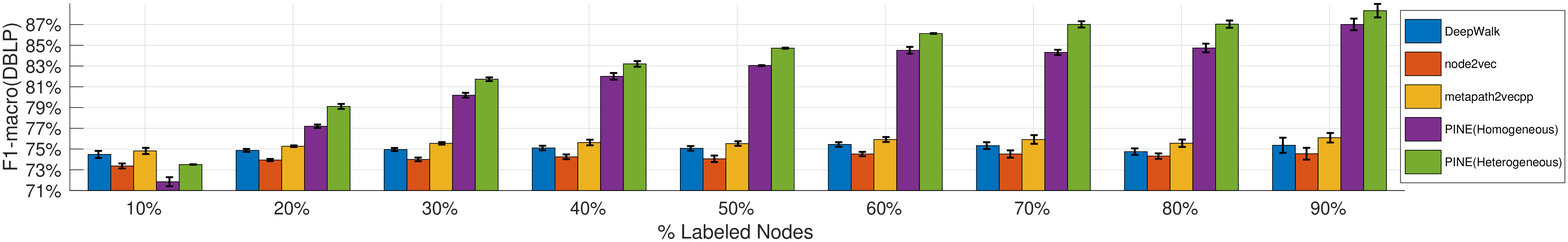}\\
        \includegraphics[width=1\linewidth]{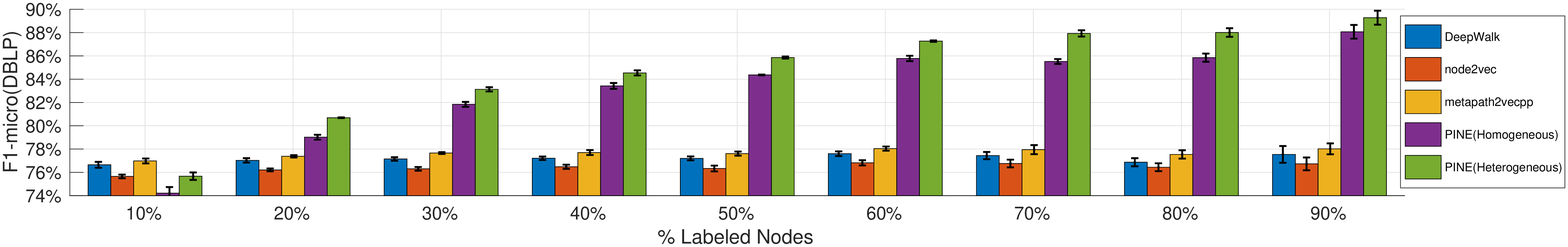}\\
          \includegraphics[width=1\linewidth]{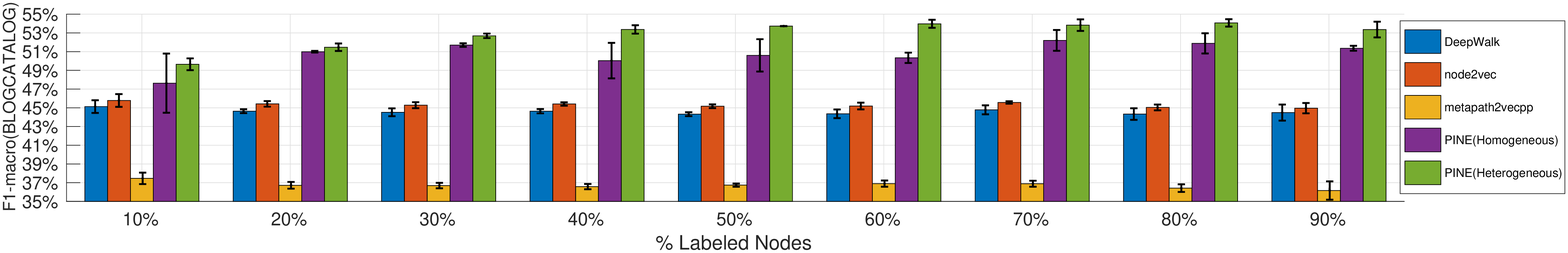}\\
          \includegraphics[width=1\linewidth]{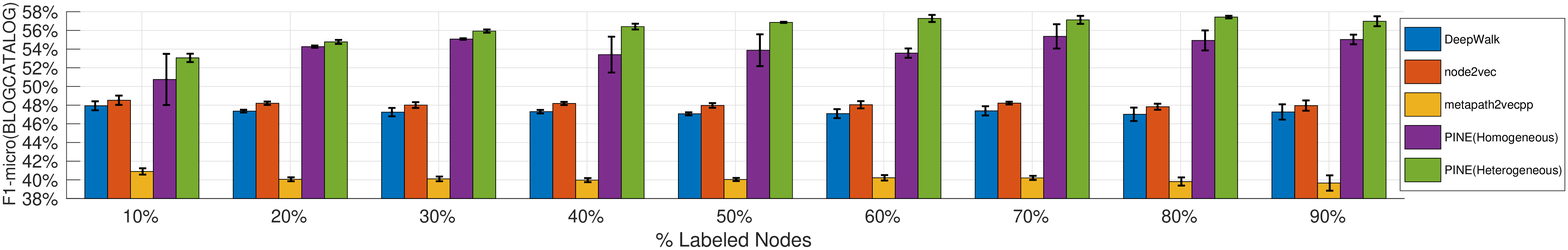}\\
          \caption{F1-score (macro, micro) (\%) of multi-label classification  in heterogeneous graphs}
          \label{fig:multilabel}
      \end{figure*}

        We can observe that in most cases, {\OM} in heterogeneous setting has the best performance. {\OM} in homogeneous setting is better than deepwalk and node2vec in  the same homogeneous setting, and is even better than \textit{metapath2vec++} in heterogeneous setting (achieving  the second best results). Overall, the superior performance of {\OM} in Figure~\ref{fig:cora-citeseer-pubmed-wiki-email} and ~\ref{fig:multilabel} demonstrates the validity of our proposed universal graph embedding mechanism.

  {\mn

  }

\section{Conclusion and Future Work}
  
    To summarize the whole paper, we propose {\OM}, a general graph embedding solution with the novel notion of partial permutation invariant set function, that in principle can capture arbitrary dependence among neighbors and automatically decide the significance of neighbor nodes at different distance for both homogeneous and heterogeneous graphs. We provide a theoretical guarantee for the effectiveness of the whole model. 
    Through extensive experimental evaluation, we show that {\OM} offers better performance on both homogeneous and heterogeneous graphs, compared to stochastic trajectories based, matrix analytics based and graph neural network based state-of-the-art algorithms.
    For the future work, our model can be extended to more general cases, e.g., involving the rich content  information out of graph neighborhood structures.


\bibliographystyle{IEEEtran}
\bibliography{bare}

\newpage
\clearpage
\appendix
{\mn
\section{Partial permutation invariant maps}
}

\subsection{Definitions: Permutation Invariant Maps and Polynomials}

\begin{definition}
\about{Symmetric group $S_N$} Given an index set $\sA = \{1,2,\cdots,N\}$,
The set of all one-to-one mappings $\pi:\sA\longrightarrow \sA$ forms the
symmetric group $S_N$ with function compositions as the group action.
\end{definition}

For brevity, we will denote $\mX\coloneqq[\bx_1, \bx_2, \cdots, \bx_N]$.
And its permutation is denoted as $T_\pi\mX \coloneqq
[\bx_{\pi(1)},\bx_{\pi(2)},\cdots, \bx_{\pi(N)}]$ for an arbitrary $\mX \in
\R^{M \times N}$, where $\pi \in S_N$, $\pi(n)\in \sA$ is the index at
position $n$ after permutation and $\bx_{\pi(n)}$ is the $\pi(n)$-th column
of $\mX$. {\mn The $T_\pi\mX$ defined here is equivalent to the permutation matrix notation we used in Definition \ref{eq.permutation_matrix_invariant}. As the proofs are based on symmetric group action, we also give the definition of permutation invariant map and partially permutation invariant map based on symmetric group action in the following.}

\begin{definition} \label{def.invariant}
\about{Permutation invariant map or $S_N$-invariant map} A continuous real
valued map $f:\bbR^{M\times N}\longrightarrow\bbR$ is permutation invariant if
\begin{equation}
f(T_\pi\mX) = f(\mX)
\end{equation}
for all $\pi \in S_N$ and all $\mX \in \bbR^{M\times N}$.
\end{definition}

\begin{definition} \label{def.partial_invaraint}
\about{Partially permutation invariant map} Given a series of symmetric group
$S_{N_1}, S_{N_2}, \cdots, S_{N_K}$, and $W \coloneqq W_1\times W_2 \cdots
\times W_K $ where $W_k \coloneqq \bbR^{{M_k}\times N_{k}}$, a continuous real
valued map $f:W\longrightarrow\bbR$ is partially permutation invariant if
\begin{equation}
f\left(T_{\pi_1}\mX_1,T_{\pi_2}\mX_2,\cdots,T_{\pi_K}\mX_K\right)
= f(\mX_1, \mX_2, \cdots, \mX_K)
\end{equation}
for all ${\pi_k} \in S_{N_k}$ and all $\mX_k \in W_k$.
\end{definition}

We call a polynomial that is a permutation invariant map \emph{permutation
invariant polynomial}. When the symmetric group involved is significant, we
will use terms like $S_N$-invariant polynomial. Similarly, a polynomial that
is a partially permutation invariant map will be called \emph{partially
(permutation) invariant polynomial}.

\subsection{Proof of \thref{theorem6}}\label{apx:theorem3-1}

A sketch of the proof ideas is as follows. Stone-Weierstrass theorem states
that polynomials are dense in the space of continuous functions. Based on this
result, we first show in \leref{le.dense} that partially permutation invariant
polynomials are dense in partially invariant continuous functional space.
Next, our main idea is to find a finite generating set for the partially
invariant functions so that any partially invariant function can be
represented as a polynomial of that generating set. There is a well-known
generating set for permutation invariant polynomials $f:\bbR^{N}
\longrightarrow \bbR$: the power sum polynomials. Polarization is introduced
to extend the generating set to the case of $f:\bbR^{M\times N}
\longrightarrow \bbR$.

\subsubsection{Polarization}

The goal of polarization is to represent polynomial invariant of a higher
dimension in terms of invariant polynomials of a lower dimension. The lemma in
the following provides an invariant representation of matrix-argument
invariant polynomial in terms of vector-argument invariant polynomial.

The group $S_N$ only permutes the columns of $\bbR^{M\times N}$. So given two
spaces $W = \bbR^{M\times N} $ and $W = \bbR^{M'\times N}$, and a linear
mapping
\begin{equation}\label{eq.linear_op}
    \mathbf{A}: W \longrightarrow W'
\end{equation}
$\mathbf{A}$ is commutable with with permutation; i.e., $\mathbf{A}(T_\pi\mX)
= T_\pi(\mathbf{A}\mX)$. If $f$ is a permutation invariant on $W'$ then
$f\circ \mathbf{A}$ is a permutation invariant on $W$. In the following lemma,
we take $M'=1$.

{\rn
\begin{lemma}\label{le.polar}
\about{Weyl's Polarization}
For any polynomial invariant $f$ on $\bbR^{M\times N}$, there exist a series
of $1\times M$ vectors $\{\ba_t\}_{t=1}^T$ and a series of $S_N$-invariant
polynomials $f_t$ on $\bbR^{1\times N}$, such that $f$ can be represented by
\[
f(\mX) = \sum_{t=1}^T f_t(\ba_t^\top \mX).
\]
\end{lemma}
}

\begin{proof}
This result is form of Weyl's poloarization \cite{weyl2016classical}. It
follows, e.g., from the Theorem 2.3 in
\cite{yarotsky2018universal} by taking the multiplicities $m'_\alpha$ of the
$\Gamma$-modules $V'_\alpha$ equal to 1, taking the group to be $S_N$, and
taking the $f_t$'s from the generating set of vector-argument invariant
polynomials.
\end{proof}

\subsubsection{Partially Permutation Invariant Polynomials}

The following lemma shows that it is sufficient to use partially invariant
polynomials to approximate partially invariant functions.

\begin{lemma} \label{le.dense}
\about{Denseness of Partially Invariant Polynomials} For any partially
invariant function $f$ on $W = W_1 \times W_2 \cdots \times W_K$ where $W_k =
\bbR^{M_k \times N_k}$ and for any $\epsilon >0$, there exists a partially
invariant polynomial $\hat{f}$ such that
$|\hat{f}(\mathcal{X})-f(\mathcal{X})|<\epsilon$ for all $\mathcal{X}\in W$.
\end{lemma}

\begin{proof}
By Stone-Weierstrass
Theorem~\cite{stone1937applications,stone1948generalized} for the compact
Hausdorff space, polynomial on compact Hausdorff space is dense in the space
of continuous functions on that same compact Hausdorff space. So for any
partially invariant function $f$ on $W$ and any $\epsilon>0$, there exists a
polynomial $f'$ on $W$ such that $|f'(\mathcal{X})-f(\mathcal{X})|<\epsilon$
for all $\mathcal{X}\in W$.

Let $\mathcal{X} = [\mX_1,\cdots,\mX_K] \in W$, and $|\mathcal{S}| =
|S_{N_1}|\cdots |S_{N_K}|$. Construct
\begin{equation} \label{eq.symf}
  \hat{f}(\mathcal{X}) = \symsum f'(T_{\pi_1}\mX_1,\cdots,T_{\pi_K}\mX_K),
\end{equation}
which is a partially invariant polynomial. We have
\small
\begin{align*}
   &|f'(\mathcal{X})-f(\mathcal{X})|\\
   &= \Big|\symsum f'(T_{\pi_1}\mX_1,\cdots,T_{\pi_K}\mX_K) - f(\mathcal{X})\Big|\\
    &= \Big|\symsum\big(f'(T_{\pi_1}\mX_1,{\tiny\cdots},T_{\pi_K}\mX_K) \tiny{-} f(T_{\pi_1}\mX_1,{\tiny\cdots},T_{\pi_K}\mX_K)\big)\Big|\\
    &\leq \symsum \left|f'(T_{\pi_1}\mX_1,{\tiny\cdots},T_{\pi_K}\mX_K) \tiny{-} f(T_{\pi_1}\mX_1,{\tiny\cdots},T_{\pi_K}\mX_K)\right| \\
    &< \epsilon
\end{align*}
\normalsize
The function in \eqref{eq.symf} thus fulfills the requirement of the lemma.
\end{proof}

The following lemma gives one form of explicit expansion for partially
invariant polynomials.
\begin{lemma} \label{le.sumofproduct}
Any partially invariant polynomial $g$ on $W = W_1 \times W_2
\cdots \times W_K$, where $W_k = \bbR^{M_k \times N_k}$, can be expressed in
the following form:
\begin{equation}\label{eq.sumofproduct}
g(\mathcal{X}) = \sum_{q=1}^Q
h_{1,q}(\mX_1)h_{2,q}(\mX_2)\cdots h_{K,q}(\mX_K).
\end{equation}
where $Q$ is an integer, and $h_{k,q}(\mX_k)$ are $S_{N_k}$-invariant.
\end{lemma}

\begin{proof}
Since $g$ is a polynomial, it is possible to write $g$ as
\begin{align*}
  g(\mathcal{X}) &= \sum_{q=1}^Q
    h'_{1,q}(\mX_1)h'_{1,q}(\mX_2)\cdots h'_{K,q}(\mX_K)
\end{align*}
where $Q$ is a suitable integer depending on the degree of $g$, and
${h'}_{k,q}$ are polynomials on $W_k$.

Define the symmetrized versions of $h'_{k,q}$ as follows:
\begin{equation}
  h_{k,q}(\mX_k) := \frac{1}{|S_{N_k}|} \sum_{\pi_k \in S_{N_k}}h'_{k,q}(T_{\pi_k}\mX_k)
\end{equation}
As $g$ is partially invariant, its value does not change if we perform
(partial) symmetrization on $g$. Therefore,
\small
\begin{align*}
&g(\mathcal{X}) =\symsum g(T_{\pi_1}\mX_1,\cdots,T_{\pi_K}\mX_K)\\
&=\symsum\sum_{q=1}^Q h'_{1,q}(T_{\pi_1}\mX_1)\cdots h'_{K,q}(T_{\pi_K}\mX_K) \\
&=\sum_{q=1}^Q h_{1,q}(\mX_1){\tiny\cdots} h_{K,q}(\mX_K)
\end{align*}
\normalsize
where the last step follows by exchanging the order of summations, and
distributing the symmetrization sums to their respective $h'_{k,q}$ functions.
\end{proof}

\begin{lemma} \label{le.hilbert}
\about{Hilbert's finiteness Theorem, e.g.,\cite{kraft2000classical}}
There exists finitely many invariant polynomials $f_1, \ldots,
f_{N_\text{inv}}: \setR^n \to \setR$ such that any invariant polynomial $f:
\setR^n \to \setR$ can be expressed as
\begin{equation}
f(\bx) = \tilde f (f_1(\bx), \ldots, f_{N_\text{inv}}(\bx))
\end{equation}
with some polynomial $\tilde f$ of $N_\text{inv}$ variables.
\end{lemma}

\begin{lemma} \label{le.powersum}
\about{Power sums as generating set} One generating set of symmetric
polynomials on $\setR^{N}$ is power sums up to degree $N$:
\begin{equation}\label{eq.gen_set}
    f_j(\bx) = \sum_{n=1}^{N} x_n^j\quad j=1,\cdots, N,\end{equation}
where $x_n$ is the $n$-th entry of $\bx$.
\end{lemma}

\subsubsection{Proof of \thref{theorem6}}

\begin{proof}
By \leref{le.dense}, any partially invariant function can be approximated by a
partially invariant polynomial, which in turn can be written in the form of
\eqref{eq.sumofproduct}, due to \leref{le.sumofproduct}. Using
\leref{le.polar}, each term (fully) invariant polynomial $h_{k,q}(\bX_k)$ can
be expressed as follows,
\begin{equation}
  h_{k,q}(\bX_k) = \sum_{t=1}^{T_{k,q}}f_{k,q,t}(\ba_t^\top \bX_k),
\end{equation}
where $f_{k,q,t}$ is an invariant polynomial.

Based on Hilbert's finiteness Theorem \leref{le.hilbert}, and the power-sum
basis result \leref{le.powersum}, each function $f_{t,k,q}(\ba_{k,q,t}^\top
\bX_k)$ is expressible as a polynomial of the following $N_k \cdot T_{k,q}$
variables:
\begin{equation} \label{eq.akqt}
  \sum_{n=1}^{N_k} (\ba_{k,q,t}^\top \bx_{k,n})^j, \quad
    t=1,\ldots, T_{k,q}, \quad j=1, \ldots, N_k.
\end{equation}

Let $\bA_{k,q}$ denote the $M_k\times T_{k,q}$ matrix whose $t$'s column is
$\ba_{k,q,t}$, $t=1, \ldots, T_{k,q}$. Define the power-sum vector function
\begin{equation}
  \bp^{(N)} (x) := [x^1, x^2, \ldots, x^N],
\end{equation}
and the function $\bq_{k,q}: \setR^{M_k}\to \setR^{N_kT_{k,q}}$
\begin{equation}
  \bg_{k,q}(\bx; \bA_{k,q}) :=[\bp^{(N_k)}(\ba_{k,q,1}^\top \bx),\; \ldots\;
    \bp^{(N_k)}(\ba_{k,q,T_{k,q}}^\top \bx)].
\end{equation}
It then follows that $f_{k,q,t}(\ba_{k,q,t}^\top \bX_k)$ is a polynomial of
\begin{equation}
  \sum_{n=1}^{N_k} \bg_{k,q} (\bx_{k,n}; \bA_{k,q}).
\end{equation}
Let $\bA_k:=[\bA_{k, 1}, \ldots, \bA_{k,Q}]$, and
\(
  \bg_k(\bx; \bA_k) := [\bg_{k,q}(\bx; \bA_{k,1}), \ldots \bg_{k,q}(\bx;
  \bA_{k,Q}]
\).
Recalling \leref{le.sumofproduct}, we establish that the function $g$ can be
approximated arbitrarily well by a function of the form
\begin{equation} \label{eq.hfunc}
 h\left(
  \sum_{n=1}^{N_1} \bg_1 (\bx_{1,n}; \bA_1),\; \ldots,\;
  \sum_{n=1}^{N_K} \bg_K (\bx_{K,n}; \bA_K)
  \right)
\end{equation}
where $h$ is a polynomial. The vector version of the result follows from the
scalar version. In the statement of the theorem, we have removed the explicit
parameters $\bA_k$'s.
\end{proof}

\subsection{Partially Permutation Invariant Neural Network: Proof of
\thref{theorem:main}}

The main idea of the proof is the following: as neural network is an universal
approximator, we use neural networks with one hidden layer to approximate $h$
in \eqref{eq.hfunc} and the power functions. We then get an approximator of
partially permutation invariant in the form of a structured neural network.

\begin{proof}
We see that \eqref{eq.hfunc} can approximate any partially permutation
invariant polynomial on $W$. By the universal approximation theorem of neural
networks \cite{cybenko1989approximation}, we can approximate the polynomial
$h$ with a shallow (one-hidden layer) neural network. Any partially
permutation invariant function can be approximated as
\begin{equation}\label{eq.multi_nn}
    \sum_{l=1}^L c_l
    \sigma\left(
      \sum_{k=1}^K\sum_{t=1}^{T_k}\sum_{j=1}^{N_k}
        w_{l,t j} \sum_{n=1}^{N_k} p_j (\ba_{k,t}^\top \vx_{k,n}) + b_l\right)
\end{equation}
where $p_j(x) = x^j$, and we have combined the double indices $q$ and $t$ of
$\ba_{k,q,t}$ in \eqref{eq.akqt} into a single index $t$ for a fixed $k$.

\noindent It is clear that the function in \eqref{eq.multi_nn} is partially
permutation invariant as all the $\{\vx_{k,n}\}_{n=1}^{N_k}$ are treated the
same. We can approximate $p_j(y)$ by $\sum_{l'=1}^{L'} d_{l'} \sigma(u_{j,l'}
y+v_{j,l'})$. Combining the two neural network approximators, it follows that
the functions $h(\cdot)$ and $\vg_k$ in Theorem~\ref{theorem6} can be chosen
in the following form:
\begin{align*}
&\rvh\left([\bz_1^\top,\cdots, \bz_K^\top]^\top=: \bar{\bz}~|~\vc, \mW\right) =  \vc^\top {\bm\sigma}(\mW\bar{\bz})
\\
&\rvg_k\left(\bx~|~T, \{\vu_t, \va_t\}_{t=1}^{T}\right) =
[
{\bm\sigma}((\vu_1 \otimes \va_{1}) \bx  + \vv)^\top, \cdots, \\ & \qquad \qquad \qquad \qquad \qquad \qquad \quad{\bm\sigma}((\vu_{T}\otimes \va_{T}) \bx+ \vv)^\top
]^\top,
\end{align*}
where we have omitted the index $k$ on $\ba_{k,t}^\top$ and used $\va$ to
denote $\ba^\top$ for simplicity.
\end{proof}

}

\section{Additional Experiment Setups}
\subsection{Configurations of Hardware and Software.} {\OM} is implemented with \texttt{PyTorch} and \texttt{TensorFlow} learning framework in the version 1.1.0 and 1.8.0 (Python 3 version). All experiments are conducted on a Linux 18.04 machine. The machine has one Core i7-6700K, 64 GB RAM, 512GB+2T Hard disks and two GTX 1080 graphics cards.

\subsection{Train-Test Splits.} For all homogeneous datasets, we conduct the same training and test splits for 5 times. In each round, we randomly sample a ratio of nodes from 10\% to 90\% to be the training set. We leave all the other nodes in the test set to evaluate the classification performance among all the methods. For the heterogeneous case, we only split the author nodes into training and test set since only author takes labels in those two datasets. We follow the same split strategy of homogeneous cases on the heterogeneous graphs as well.

\end{document}

%% file: math_commands.tex
\usepackage{amsmath,amsfonts,bm}

\def\secref#1{section~\ref{#1}}

\def\1{\bm{1}}

\def\rc{{\textnormal{c}}}

\def\rn{{\textnormal{n}}}

\def\rvf{{\mathbf{f}}}
\def\rvg{{\mathbf{g}}}
\def\rvh{{\mathbf{h}}}

\def\rvt{{\mathbf{t}}}

\def\va{{\bm{a}}}
\def\vb{{\bm{b}}}
\def\vc{{\bm{c}}}

\def\vg{{\bm{g}}}

\def\vu{{\bm{u}}}
\def\vv{{\bm{v}}}
\def\vw{{\bm{w}}}
\def\vx{{\bm{x}}}
\def\vy{{\bm{y}}}

\def\evb{{b}}
\def\evc{{c}}

\def\evu{{u}}
\def\evv{{v}}
\def\evw{{w}}

\def\mA{{\bm{A}}}
\def\mB{{\bm{B}}}

\def\mP{{\bm{P}}}

\def\mW{{\bm{W}}}
\def\mX{{\bm{X}}}

\DeclareMathAlphabet{\mathsfit}{\encodingdefault}{\sfdefault}{m}{sl}
\SetMathAlphabet{\mathsfit}{bold}{\encodingdefault}{\sfdefault}{bx}{n}

\def\gE{{\mathcal{E}}}

\def\gG{{\mathcal{G}}}

\def\gV{{\mathcal{V}}}

\def\gX{{\mathcal{X}}}

\def\sA{{\mathbb{A}}}

\def\sP{{\mathbb{P}}}

\def\sR{{\mathbb{R}}}

\newcommand{\R}{\mathbb{R}}

